\documentclass{ecai} 

\usepackage{latexsym}
\usepackage{amssymb}
\usepackage{amsmath}
\usepackage{amsthm}
\usepackage{booktabs}
\usepackage{enumitem}
\usepackage{graphicx}
\usepackage{color}

\usepackage{natbib}

\usepackage{caption}
\usepackage[separate-uncertainty=true]{siunitx}
\usepackage{mathtools}
\usepackage{graphicx}
\usepackage{courier}
\usepackage{epstopdf}
\usepackage{color}
\usepackage{csquotes}

\usepackage{amsmath}
\usepackage{amsfonts}
\usepackage{amssymb}
\usepackage[subnum]{cases}
\usepackage{booktabs} 
\usepackage{pifont}

\usepackage{algorithm,algorithmicx,algpseudocode}

\usepackage{amssymb}
\usepackage{pifont}

\usepackage {tikz}
\usetikzlibrary {positioning}
\definecolor {processblue}{cmyk}{0.96,0,0,0}
\usepackage{lipsum,adjustbox}

\usepackage[subnum]{cases}
\usepackage{color,soul}

\usepackage{subcaption}

\usepackage{footnote}
\makesavenoteenv{tabular} 

\usepackage{multirow}

\usepackage{booktabs,subcaption,amsfonts,dcolumn}
\newcolumntype{d}[1]{D..{#1}}

\usepackage{xcolor}

\usepackage{textcomp}

\usepackage{bbold}
\usepackage{mathtools}
\usepackage{breqn}

\usepackage{amsthm}
\newtheorem{theorem}{Theorem}[section]

\usepackage[many]{tcolorbox}
\newtcolorbox{boxA}{
    fontupper = \bf,
    boxrule = 1.5pt,
    colframe = black 
}

\newcommand{\BibTeX}{B\kern-.05em{\sc i\kern-.025em b}\kern-.08em\TeX}

\begin{document}


\begin{frontmatter}


\paperid{1656} 


\title{PostDoc: Generating Poster from a Long Multimodal Document Using Deep Submodular Optimization}


\author[A]{\fnms{Vijay}~\snm{Jaisankar}}
\author[B]{\fnms{Sambaran}~\snm{Bandyopadhyay}\thanks{Corresponding Author. Email: samb.bandyo@gmail.com}}
\author[C]{\fnms{Kalp}~\snm{Vyas}} 
\author[C]{\fnms{Varre}~\snm{Chaitanya}\thanks{Vijay Jaisankar, Kalp Vyas and Varre Chaitanya were interns at Adobe Research when the work was conducted.}}
\author[B]{\fnms{Shwetha}~\snm{Somasundaram}}

\address[A]{International Institute of Information Technology, Bangalore}
\address[B]{Adobe Research}
\address[C]{IIT Bombay}


\begin{abstract}
A poster from a long input document can be considered as a one-page easy-to-read multimodal (text and images) summary presented on a nice template with good design elements. Automatic transformation of a long document into a poster is a very less studied but challenging task. It involves content summarization of the input document followed by template generation and harmonization. In this work, we propose a novel deep submodular function which can be trained on ground truth summaries to extract multimodal content from the document and explicitly ensures good coverage, diversity and alignment of text and images. Then, we use an LLM based paraphraser and propose to generate a template with various design aspects conditioned on the input content. We show the merits of our approach through extensive automated and human evaluations.
\end{abstract}

\end{frontmatter}



\section{Introduction}\label{sec:intro}
Recent success of large language models \citep{openai2023gpt4,petroni2020how} and large vision language models \citep{radford2021learning,li2022blip} have led to several new applications in the field of generative AI. 
In this work, we focus on transforming a long multimodal document, containing both text and images, into a poster, which is a visually rich one-page multimodal summary of the document. A poster should have a good coverage of the overall content of the document and, is generally easy-to-read and presented in a nice template with good design elements \citep{shelledy2004make}. 
To create a poster, there are different types of design tools and products available such as Microsoft PowerPoint~\footnote{\url{https://create.microsoft.com/en-us/templates/posters}} and Adobe Express~\footnote{\url{https://www.adobe.com/express/}}.
However, these products need significant manual effort to select the multimodal content from a document and choose the appropriate design elements in the poster. Such manual efforts are time consuming and often need specific domain expertise based on the document.

Automatic transformation of a document to a poster is a less studied problem in the research literature and industry \citep{qiang2019learning,xu2021neural,xu2022posterbot}. Such a transformation process mainly involves two major steps: (i) Content summarization (or planning), which aims to select key content from the document and paraphrase them in some appropriate format to be put in the poster and  (ii) Template generation and harmonization, which involves generating a suitable layout and design elements such as background of the poster, font and size of letters, number of text and image elements etc. of the poster based on the output of the content planning step and fill up the generated template with the generated content.
Content summarization for poster is challenging because of the following reasons. A poster is very limited in size (single large page), but the input document can have multiple pages. Thus, it is essential that the content in the poster (i) represents all the important aspects of the input document (coverage), (ii) has very less repetition within it (diversity) and (iii) has well aligned images and text. Coverage and diversity have been studied in text summarization literature \citep{lin-bilmes-2011-class}. But their interpretation in the multimodal setup (with text and images) is not well-understood. 

Recent advent of zero-shot and few-shot LLMs, especially GPT-3.5-turbo (ChatGPT) and GPT-4 \citep{openai2023gpt4} have significantly improved the state-of-the-art (SOTA) for several natural language processing tasks like
summarization. However, using them directly for the content summarization of poster generation is not feasible because: (i) Current LLM models (including GPT-4) cannot produce multimodal output and their ability to process multimodal input is still questionable; (ii) Since we focus on long documents as input, LLM based algorithms cannot process that lengthy text at single shot. There are techniques which divide long text into chunks and feed each chunk separately to an LLM for summarization \citep{bhaskar2023prompted}. Researchers have also tried to improve the context length of LLMs using length extrapolation \citep{peng2023yarn} and position interpolation of transformers \citep{chen2023extending}. However, such approaches lose the global view of the document, computationally expensive or cannot be used with black-box LLMs. Moreover, LLMs tend to hallucinate and their performance drops when the input context is very long \cite{liu2023lost}.

In this paper, we have addressed all the challenges mentioned above by proposing an approach to handle multimodal content jointly and avoid feeding the entire content to an LLM directly. 
%
Following are the key contributions made in this work:
\textbf{(1)} We propose an efficient and computationally fast end-to-end pipeline, referred as \textit{PostDoc} (in Figure \ref{fig:PostDoc}), for automatically generating a visually rich \textit{Post}er from a long multimodal input \textit{Doc}ument.
\textbf{(2)} For selecting suitable content from the input document, we propose a novel optimization formulation using deep submodular functions which explicitly ensures coverage, diversity and multimodal content alignment in the multimodal summary. This summary is passed to an LLM (GPT-3.5-turbo, a.k.a., ChatGPT) for paraphrasing to put into a poster.
\textbf{(3)} Based on the input document and the selected content, we generate a suitable poster template with different design elements to ensure that the poster looks good aesthetically. 
\textbf{(4)} We conduct thorough experimentation to validate the quality of the content and design aspects of the generated posters through automated and human evaluations.

\section{Related Work and Background}\label{sec:related}

\subsection{Multimodal Summarization} 
Text summarization has been studied quite extensively in the literature \cite{zhang2020pegasus,zhong2020extractive}. From the last few years, researchers have focused into multimodal summarization which involves different modalities such as text, images and videos, and leverages cross-modal information \cite{mademlis2016multimodal,li2017multi}. In this section, we focus on works which consider both input and output to be multimodal, as required for transforming a document to a poster. 
\citet{zhu-etal-2018-msmo} generated abstractive multimodal summary from a document, but it is trained by the target of text modality, leading to the modality-bias problem.
\citet{Zhu_Zhou_Zhang_Li_Zong_Li_2020} extended it by including images along with the text but is currently limited to generating summaries with only one image. \citet{zhang2022unims} further improved on these methods by using BART and knowledge distillation which does a better job of image selection but still treats image and text as different entities. \citet{He_2023_CVPR} presented A2Summ, a novel unified
transformer-based framework for multimodal summarization primarily focusing on  text and video modalities. \citet{zhang2021hierarchical} proposed a technique to summarize the document based on using the hierarchy in the form of a graph network but also faces the issue of choosing only one image for the summary. 


Submodular functions have been used for multiple summarization tasks \cite{lin-bilmes-2011-class,tschiatschek2014learning,modani2016summarizing,zhu-etal-2018-msmo} as they are a natural fit for text summarization.
A set of differentiable submodular functions, also known as deep submodular functions \cite{bilmes2017deep,kothawade2020deep} have been proposed in the literature and also used for text summarization. However, they have not incorporated the multimodal aspects such as multimodal coverage, diversity, and image-text alignment terms which are key components for multimodal summarization.
In this work, we address this research gap and aim to design a deep submodular function which captures a set of intrinsic properties of multimodal extractive summarization and also trainable from the ground truth data.
 
\subsection{Layout Generation}
Regarding the design aspects, \citet{qiang2019learning} address poster layout generation using recursive divisions, yet within the context of scientific papers. \citet{Gupta_2021_ICCV} propose a layout generation method. However, it lacks conditional generation. On the other hand, \citet{chai2023layoutdm} generate layouts conditionally based on specific requirements, such as the count and placement of text and image bounding boxes. 
\citet{hsu2023posterlayout} generate layout depending on the structure of the background image being used but this method too lacks conditional generation. These methods are computationally expensive, but not effective in generating posters with non-overlapping bounding boxes. 
Our work addresses these limitations by adopting a simple heuristic-based layout generation approach specifically for posters. Additionally, we enhance poster aesthetics by considering design elements such as colors, fonts, and backgrounds, based on the content. 

\subsection{Document Transformation}
There are few existing works related to poster creation from documents. \citet{xu2021neural} generates posters from documents but relies on template retrieval from a fixed set of templates and also limiting its applicability to research papers exclusively. There are few research works on generating slides from scientific documents automatically \citep{sun2021d2s,fu2022doc2ppt}. But they often need users to come up with an outline specific to slide presentation or summarize individual sections in each slide.

\subsection{Background on Submodular Functions}
Here, we discuss some key concepts required to understand our solution. A submodular function $f$ is a set function with diminishing returns property. In simple terms, adding an element to a smaller set provides a larger marginal gain compared to adding the same element to a larger set. Mathematically, for sets $A \subseteq B$, on adding extra element $x \notin B$, $f(A \cup {x}) - f(A) \geq f(B \cup {x}) - f(B)$ \citep{lin-bilmes-2011-class}. A key advantage of using a submodular function is that there exists a simple greedy algorithm of iteratively choosing the element that maximises the marginal gain with an approximation guarantee.

Recently, researchers explore deep (trainable) submodular functions \citep{bilmes2017deep,kothawade2020deep} where the data is projected into a suitable feature or embedding space. They can be represented as: 
$f(A)=\sum_{u \in \mathbb{U}} \Phi\Big(w(u)m_u(A)\Big)$, where $\Phi$ is a non decreasing non-negative concave function, $w(u)$ represents the trainable weight of the feature $u$, $m_u(S) = \sum_{s \in S} m_u(s)$ is a non-negative modular function. These functions enable the learning of submodular functions through a training dataset and also their inference is fast compared to the quadratic complexity of most of the other submodular functions.





\section{Problem Statement and Solution Approach}\label{sec:prob}
As discussed in Section \ref{sec:intro}, we aim to automatically generate a poster from a long document.
The document may contain different types of multimodal content such as text, images, tables, charts, etc. For the ease of presentation, we use image to represent all such non-textual elements in a document. As shown in Figure \ref{fig:PostDoc}, we use Adobe Extract API~\footnote{\url{https://developer.adobe.com/document-services/apis/pdf-extract/}} to extract the multimodal content from the document.
Then, we use the pre-trained multimodal model BLIP \citep{li2022blip} to encode both text and image elements into a common vector space of dimension 768. 
We have observed that the embeddings of text and images are often not in the same scale. 
To overcome this issue, we first shift all the embeddings to positive coordinate of the embedding space and do a L1 normalization of the embeddings. 

With this, we present the problem mathematically as follows.
Given $D = (e_1, e_2, \cdots, e_N)$ is a multimodal input document with $N$ (can vary over the documents) content elements in order where each content element can be a text sentence or an image. We assume that each $e_i \in \mathbb{R}^d_+$ denotes a $d$ dimensional normalized BLIP embedding (as discussed above) of the $i$th content element in the document ($d=768$ in this case). Our goal is to select a subset of content elements $A \subseteq D$ with $|A| \leq K$ so that content elements of $A$ have good coverage, diversity and alignment as discussed in Section \ref{sec:intro}. Since the size of the poster is limited, we extract the corresponding maximum $K>0$ content elements from the document. We discuss how to fix $K$ in Section \ref{sec:exp}.
For training our content selection algorithm, we use a training set of documents with their ground truth summary. So, we assume to have the following training set $\mathcal{T} = \{(D_1,A^*_1), (D_1,A^*_1), \cdots, (D_M,A^*_M)\}$ containing $M$ pairs of documents and the corresponding ground truth summary.
During the inference for a given document $D$, we will first select the subset $A$ using the trained content selection model. Then we will paraphrase $A$ into a poster friendly format. We also generate a poster template with suitable design elements based on the content and put the paraphrased content into it.

\begin{figure*}
    \centering
    \includegraphics[width=0.8\textwidth]{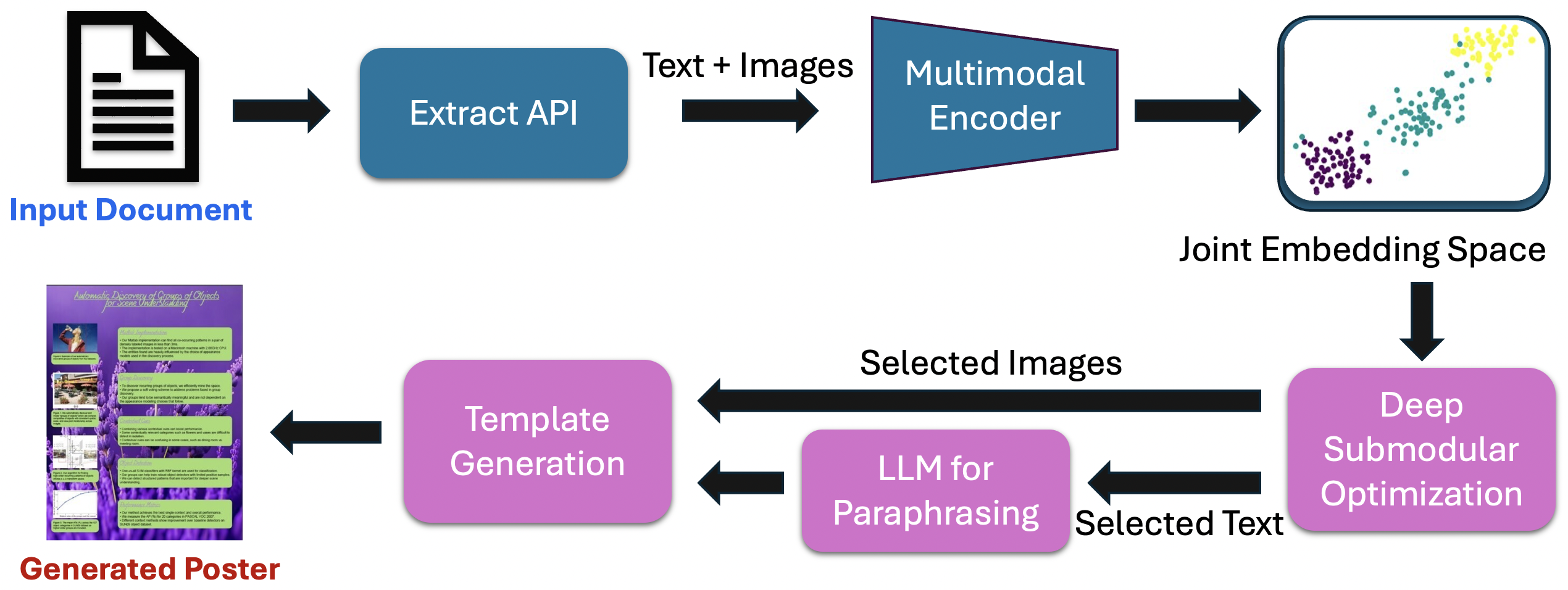} 
    \caption{Block Diagram of PostDoc}
    \label{fig:PostDoc}
\end{figure*}


\subsection{Multimodal Extractive Summarization}\label{sec:conPlan}
Once the embeddings of all the content elements are obtained from a document, the next step is to select only a subset of them such that some desired properties are satisfied. We propose a novel deep submodular based optimization framework for this task.
As mentioned in Section \ref{sec:prob}, the normalized BLIP embeddings of the sequence of content elements (text sentences or images) from the document are represented as $D = (e_1, e_2, \cdots, e_N)$, with $e_i \in \mathbb{R}^d_+$. In this subsection, our goal is to extract a subset of content elements $A \subseteq D$ with $|A| \leq K$ such that the following properties are preserved in the extracted subset: (1) \textbf{Coverage}: We want the content elements of $A$ to represent the whole document $D$ well. Thus, it is expected that any content element in $A$ to be more similar to many elements in $D$. (2) \textbf{Diversity}: Since the poster is of very limited size, we want to avoid unnecessary redundancy in content present in the poster. This property ensures that any content element in $A$ is very different from most of the other elements in $A$. (3) \textbf{Multimodal Alignment}: The output poster contains both images and text. Images and text contain complementary information. Hence, it is important to ensure that the images and text present in the poster are aligned with each other. For e.g., the text present in the poster can brief about the image. Thus, any image element in $A$ should be similar to some text elements in $A$ and vice-versa. (4) \textbf{Ground Truth Data Adaptability}: All the above properties are different intrinsic properties expected in a summarization. However, ground truth summarization data may have other hidden properties which may not be captured above. So, we also want our algorithm to learn from the set of multimodal ground truth summary data. 

Thus, $A$ can be considered as an extractive multimodal summary of the multimodal document $D$ where our goal is to design a summarization model which is trainable and equipped with the inductive bias as mentioned above. Next, we try to construct an objective function to achieve this.
\begin{dmath}\label{eq:dsf}
    f(A) = \sum_{u \in [d]} w_u \;\;
    \sqrt{
        \begin{aligned}
            \sum_{x \in A}\sum_{y \in D} x_u y_u 
            - \sum_{x \in A}\sum_{y \in A} x_u y_u \\ 
            + \sum_{x \in A_I}\sum_{y \in A_T} x_u y_u 
            + |D| \sum_{x \in A} x_u
        \end{aligned}
    }
\end{dmath}
Here, $[d] = \{1,2,\cdots,d\}$ denotes the set of dimensions of $\mathbb{R}^d$. We use a vector of trainable weight parameters $w = [w_1,w_2,\cdots,w_d]^T \geq 0$ which intuitively captures the importance of each dimension of the embedding space. Ideally for a given document $D$, we would like to choose the subset $A \subseteq D$ which maximizes $f(A)$. The term $\sum_{x \in A}\sum_{y \in D}x_u y_u$ captures the similarity of an element $x \in A$ to an element ${y \in D}$ for the $u$th dimension. Since both $x$ and $y$ are L1-normalized, this term over the outer summation contributes more when $x$ and $y$ are similar to each other. Thus, the first term captures the notion of coverage. Similarly, the second term $\sum_{x \in A}\sum_{y \in A} x_u y_u$ captures the similarity of content elements within the extracted multimodal summary $A$. Thus, the negation of this term can be considered as the diversity of the elements within $A$. So far, we have not differentiate between the text and images within $A$. However, selected images in $A$ need to be aligned with the text present in the poster. The third term $\sum_{x \in A_I}\sum_{y \in A_T} x_u y_u$ in Equation \ref{eq:dsf} ensures multimodal alignment, where $A_I$ is the set of images in $A$ and $A_T$ is the set of text sentences in $A$. The last term $|D| \sum_{x \in A} x_u$ is introduced to ensure some nice mathematical property of our loss function. Since, $x \in D$ are L1-normalized, the last term is a constant if $w_u=\frac{1}{d}$, $\forall u \in [d]$ (initial condition as discussed in Section \ref{sec:opti}).

Next, we want to design a loss function to train the parameters $w$ of our model. As mentioned in Section \ref{sec:prob}, we are given with a training set of multimodal documents with the corresponding ground truth summaries as $\mathcal{T} = \{(D_1,A^*_1), (D_1,A^*_1), \cdots, (D_M,A^*_M)\}$. For any $(D_i,A^*_i) \in \mathcal{T}$, we want the value of the function $f$ on our model generated summary to be close to $f(A^*_i)$. We consider the following hinge loss function here:
\begin{dmath}\label{eq:loss}
    \min_{w \geq 0} \;\; \sum_{i=1}^M \bigg( \max\Big( \max_{\substack{A \subseteq D_i \\ |A| \leq K}} \{f(A)\} - f(A_i^*), 0 \Big) + \frac{\lambda}{2} || w ||^2_2 \bigg)
\end{dmath}
In the above equation, we also use an L2 regularizer on $w$ with a weight hyperparameter $\lambda \geq 0$. Based on the performance on validation set, we keep $\lambda=0.1$ for all our experiments. The predicted summary is obtained by maximizing $f(A)$ w.r.t. $A$ such that $A \subseteq D$ with $|A| \leq K$. By minimizing the hinge loss w.r.t. the trainable non-negative weight parameters $w$ ensures that the maximum of $\{f(A)\}$ (i.e., on model predicted summary) is not too far less from the ground truth summary on the training data. We discuss the solution strategy and training of this optimization problem next.

\subsection{Training and Optimization}\label{sec:opti}
The optimization function in Equation \ref{eq:loss} is a constrained min-max optimization on two different types of variables. Here, $w$ is continuous and non-negative. But $A$ is a subset of with a fixed cardinality. To solve this, we use an iterative alternating optimization strategy as discussed below. 

\subsubsection{Maximization w.r.t. $A$}
Let us first focus on maximizing the objective w.r.t. $A$ while keeping $w$ fixed. Then it is essentially a subset selection problem for each $D_i$, $\forall i=1,2,\cdots,M$. Subset selection problems are typically combinatorial in nature and often computationally infeasible. But the following theorem shows an important property of $f$ which will help us to solve the optimization problem.

\begin{theorem}\label{thm:submod}
    \textbf{The set function $f$ in Equation \ref{eq:dsf} is a monotone submodular function.}
\end{theorem}
\begin{proof}
Let the image-text alignment term of $f(A)$ be
$h(A) = \sum_{x \in A_I}\sum_{y \in A_T}x_uy_u$.
The function $f$ in Equation \ref{eq:dsf} can be simplified to $f(A) = \sum_{u \in [d]} w_u\sqrt{g(A)}$ where,
\begin{equation}
    \label{eq:g(A)}
    g(A) = \Bigl( \sum_{x \in A}x_u \Bigr) \Biggl( |D| + \sum_{y \in D, y \notin A}y_u \Biggr) + h(A)
\end{equation}

In order to show that $f(A)$ is monotone submodular, we can show that $g(A)$ is monotone submodular function.
\[ g(A \cup \{p\}) = \Bigl( \sum_{x \in A}x_u + p_u \Bigr) \Biggl( |D| + \sum_{y \in D, y \notin A}y_u \ \ - \ \ p_u\Biggr) + h(A\cup\{p\}) \]
The proof will remain same for either $p \in A_I$ or $p \in A_T$. So let us consider the case $p \in A_I$.
\[ h(A \cup \{p\}) \ - \ h(A) = p_u\sum_{y \in A_T}y_u\]
\begin{equation}
    \label{eq:marginal g(A)}
    g(A \cup \{p\}) \ \ - \ \ g(A) =  p_u \biggl( |D| +  \sum_{y \in D}y_u  +  \sum_{y \in A_T}y_u -  \ \ 2 \sum_{x \in A}x_u  \biggr)
\end{equation}

For simplicity, lets denote $g(A \cup \{p\}) -  g(A)$ as $g_A(p)$
Since \[|D| > \sum_{x \in A}x_u \ \ and \ \  \sum_{y \in D}y_u  > \sum_{x \in A}x_u  \]
\[ g_A(p) =  p_u \biggl( \bigl(|D| -  \  \sum_{x \in A}x_u\bigr) +  \bigl(\sum_{y \in D}y_u -  \  \sum_{x \in A}x_u\bigr)  +  \sum_{y \in A_T}y_u   \biggr) > 0\]

Hence $g(A)$ is monotone. For submodularity, the property of diminishing marginal gains should hold. Let $A \subset B$. We need to show that $g_A(p) \geq g_B(p)$.
Since the term
$p_u \biggl( |D| +  \sum_{y \in D}y_u ) \biggr)$
in the equation \ref{eq:marginal g(A)} is independent of A, showing $g_A(p) \geq g_B(p)$ dissolves to showing
\begin{equation}
    \label{eq:ineq}
    \sum_{y \in A_T}y_u -  \ \ 2 \sum_{x \in A}x_u \geq \sum_{y \in B_T}y_u -  \ \ 2 \sum_{x \in B}x_u
\end{equation}
since, 
$\sum_{y \in A_T}y_u - 2 \sum_{x \in A}x_u = - \Bigl( \sum_{y \in A_I}y_u + \sum_{x \in A}x_u \Bigr)$.
As A increases the value becomes more negative, hence the inequality \ref{eq:ineq} is always valid. This proves the submodular property of $g(A)$.
Hence g(A) is monotone-submodular. Since we know that for any monotone submodular function $p(A)$,  $q(A)=r(p(A))$ is monotone submodular, if $r$ is any non decreasing concave function \citep{lin-bilmes-2011-class}, We can say that the set function $f(A)$ is a monotone submodular function.
\end{proof}

Please note that the submodular function $f$ in Equation \ref{eq:dsf} proposed by us for multimodal summarization is very different from the existing set of submodular functions used for the purpose of text summarization. Most of those existing functions capture diversity rewards assuming the availability of cluster of text and ignore the multimodal aspect of the problem \citep{lin-bilmes-2011-class,modani2016summarizing}.
To maximize $f(A)$, we use the simple greedy algorithm which is a $(1-\frac{1}{e})$ factor approximation of the optimal solution since $f$ is monotone and submodular \citep{lin-bilmes-2011-class}. At each step of the greedy algorithm, we include a new content element $x \in D \setminus A$ into $A$ for which $f(A \cup \{x\}) - f(A)$ is maximum till $|A| = K$.

\subsubsection{Minimization w.r.t. $w$}
Next, we assume $A$ to be fixed and use a projected stochastic gradient descent approach to minimize the loss function in Equation \ref{eq:loss} w.r.t $w \geq 0$. The subgradient of the loss w.r.t. $w_u$ ($u$th dimension) on the $i$-th sample of the training set is calculated as $\frac{\partial f(A)}{\partial w_u} - \frac{\partial f(A_i^*)}{\partial w_u} + \lambda w_u$, where $A = \max_{\substack{A \subseteq D_i \\ |A| \leq K}} \{f(A)\}$.
This gives us the stochastic gradient descent step with a learning rate $\alpha > 0$ as (check proof of Theorem \ref{thm:submod}):
\begin{align*}
    w_u = w_u - \alpha \Big( \sqrt{\Bigl( \sum_{x \in A}x_u \Bigr) \Biggl( |D| + \sum_{y \in D, y \notin A}y_u \Biggr) + h(A)} \\ - \sqrt{\Bigl( \sum_{x \in A_i^*}x_u \Bigr) \Biggl( |D| + \sum_{y \in D, y \notin A_i^*}y_u \Biggr) + h(A_i^*)} + \lambda w_u \Big)
\end{align*}
To ensure the non-negativity, we project it to the positive quadrant by setting $w_u = \max(0, w_u)$, $\forall u=1,2,\cdots,d$.

We iteratively solve the optimization problem in Equation \ref{eq:loss} by maximizing it w.r.t. $A$ by keeping $w$ fixed, and then minimizing it w.r.t. $w$ by keeping $A$ fixed. We repeat these two steps until the loss converges on the validation set used in the experiments.

To calculate the value of $f(A)$, we use a weighted sum of $d$ terms and each term can be calculated using pre computing and using the previous term. Since we are using the greedy algorithm to find $max_{A \subseteq 
D_i} {f(A)}$ , this total step has time complexity of $O(KNd)$. For the update of the weights step, we do $d$ updates (all weights get updated) and for each update, all the summations can be done independently and then multiplied so we get the time complexity for each update as $O(K)$ ($|D|$ and total sum ($\Sigma_{y \in D} y_u$) values are pre computed hence the only term calculated during run time is $\Sigma_{x \in A} x_u$ term, which is then used to calculate $\Sigma_{y \in D, y \notin A} y_u = \Sigma_{y \in D} y_u - \Sigma_{x \in A} x_u$), this gives us total training time complexity as $O(NKd + Kd)$ for one full training update. For the inference time, we only need to use greedy algorithm to get the best possible subset (summary), hence this is will be of time complexity $O(KNd)$


\subsection{Content Paraphrasing}
Text sentences in the multimodal extractive summary may not be suitable to put in the poster directly. 
We choose ChatGPT (GPT-3.5-turbo) as it is shown to perform very well for text paraphrasing for different use cases \citep{chen2023robust}.
But, applying it directly to the long text of the whole document is not possible due to the length of the document. The approaches that deal with the context length has their limitations as discussed in Section \ref{sec:intro}.
However, the length of the extractive summary is limited to $K$. We choose $K$ in such a way that the whole text from the extracted multimodal summary can be fed to ChatGPT within a single API call. A sample prompt for paraphrasing the text in the poster is:\\
\textit{"Group and rephrase the content of the following text into 5 to 8 topics without altering the order such that for each topic, there is a title and atleast 3 rephrased sentences as bullet points so that it will look good in a poster. Do not add any new content.\\
Text: \{Extractive Summary Text\}"}


We use the output paraphrased text along with the images selected in the multimodal extracted summary as the content to be put in the poster.

\subsection{Template Generation}\label{sec:tempgen}
We define the template to be a combination of different style elements such as font of the text and different colors to be used, and the layout of the poster (position of different content elements). We discuss each of them as follows.

\subsubsection{Font Selection}

Fonts are crucial components of posters as they signal the intent of the content provided and enable mental maps for familiarity.
For example, cooking books are often associated with serif and cursive styles. 
In this regard, we train a model based on the poster title guiding the visual attributes of the font chosen.
The dataset for this task was the \textit{Let me choose} dataset \citep{shirani2020font}, which consists of pairs of titles and the most appropriate fonts for them, from a sample size of $10$ fonts.
%
We use a fine-tuned MiniLM \citep{DBLP:journals/corr/abs-2002-10957} transformer model as the base encoder to the font selection model. The 384-dimensional feature vector is passed through a 2-layer fully connected network with a dropout layer and uses the LeakyReLU activation function. 
To find the appropriate learning rate for this process, we use LR-Finder~\footnote{\url{https://github.com/davidtvs/pytorch-lr-finder}} and trained the model for 20,000 epochs on the \textit{train} section of the dataset.


\begin{figure}
    \centering
    \includegraphics[width=0.27\textwidth]{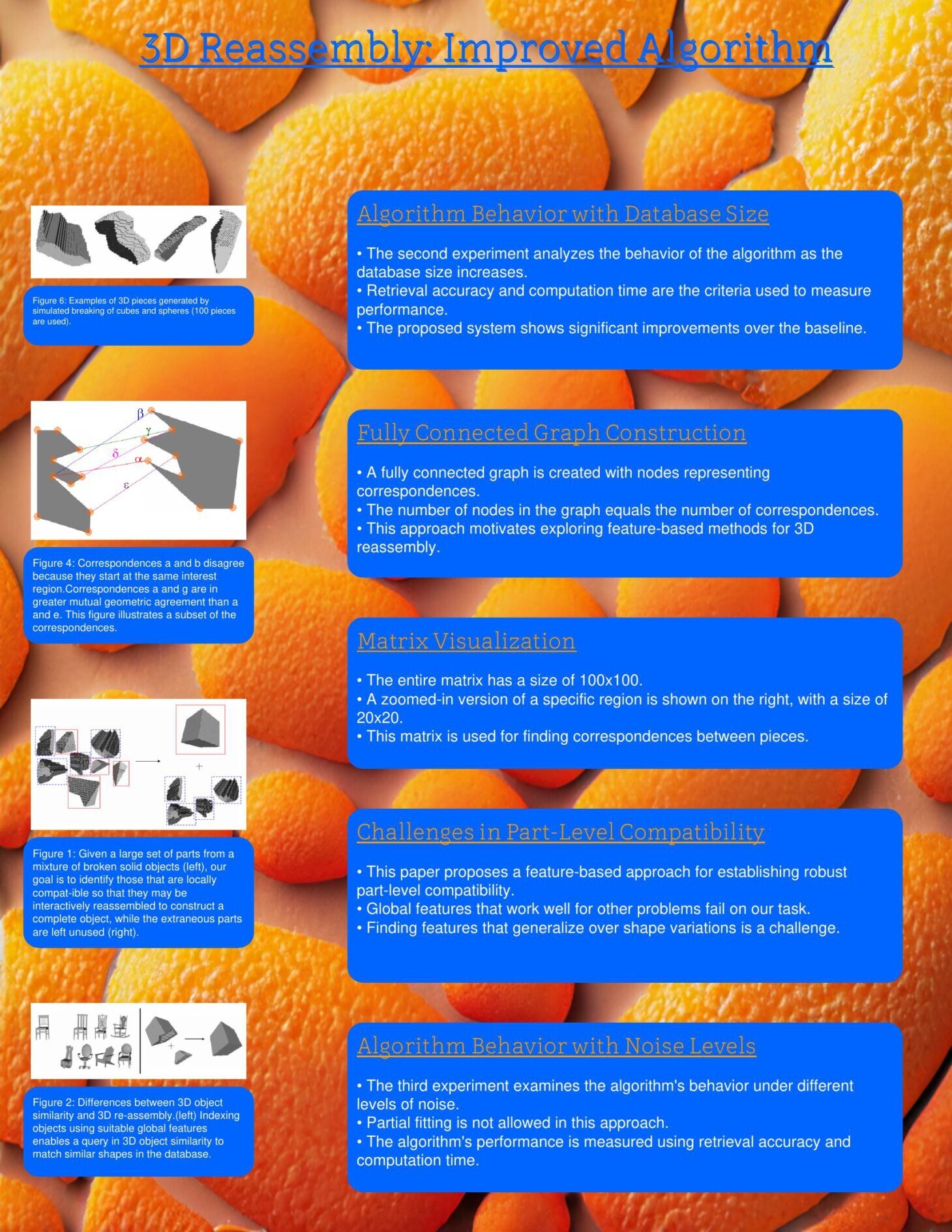} 
    \caption{A sample poster generated by PostDoc for a research paper}
    \label{fig:sample_poster}
\end{figure}

\subsubsection{Color Selection}

Colors are also important for posters as they capture attention and contextualise the content.
Our pipeline has three main colors: (1) Background color of the poster; (2) Box fill color that serves as the color for the bounding boxes that house textual content; and (3) Text fill color - the font color of the texts.
To generate the colors used in the poster, we use a model trained on the TPN architecture as proposed in Text2Colors \citep{text2colors}. This model, called \textit{TPNSmall}, is trained for 7,000 epochs on the \textit{Text2Colors} dataset. For a given poster title, we first pass it through a QA (Question Answering) model with the prompt \textit{"What is the main point here?"}. This gives the \textit{intent} of the poster. TPNSmall outputs a palette of hex codes, given the intent.
We make use of a publicly available fine-tuned Question Answering model~\footnote{\url{https://huggingface.co/deepset/roberta-base-squad2}}. 
The \textit{dominant color} in the palette is chosen to be the background color of the poster, and its \textit{complement} is chosen to be the box fill color. The text text fill color is chosen to be black or white based on its contrast with the box fill color. We then use the background color in a prompt to Firefly~\footnote{\url{https://firefly.adobe.com/}} that generates a background grounded on it.

\subsubsection{Layout Generation}
\label{sec:layout_pipeline}
We propose a heuristic based approach for generating a balanced layout conditioned on the paraphrased content. For each topic with the associated bullet points from the paraphrased content, we create a text box in the poster layout. Similarly, for each image (with the associated caption when available), we create an image box. 
We kept the images on the left and text boxes on the right by dividing the space into two parts vertically. The width of the text boxes is fixed where as that of the images is adjusted based on number of images.  
To estimate the height of the text boxes, we consider the content length.
The detailed calculations are provided in the supplementary material. By following this approach, we achieve a well-organized and visually appealing layout for posters, adapting the design based on the number of text and image boxes required. A sample poster is shown in Figure \ref{fig:sample_poster}.

\begin{table*}[ht]
\centering
\caption{Comparison of various multimodal summarization methods on MSMO dataset ~\citep{zhu-etal-2018-msmo}}
\resizebox{0.9\linewidth}{!}{
\begin{tabular}{c|cccccccc}
    \toprule
      Method & ROUGE-L & ROUGE-1 & ROUGE-2 & Coverage & Diversity & Image Precision & Inference time (sec)\\
    \midrule
    MemSum + BLIP & 0.24 $\pm$ 0.11 & 0.37 $\pm$ 0.12 & 0.15 $\pm$ 0.11 & 0.29 $\pm$ 0.06 & 0.31 $\pm$ 0.11 & 0.73 $\pm$ 0.33 & 3.27  \\    BRIO + BLIP & 0.36 $\pm$ 0.11 & 0.43 $\pm$ 0.11 & 0.18 $\pm$ 0.10 & 0.37 $\pm$ 0.06 & 0.45 $\pm$ 0.20 & 0.75 $\pm$ 0.32 & 1.04  \\
    GPT-3.5 + BLIP & 0.27 $\pm$ 0.08 & 0.34 $\pm$ 0.08 & 0.11 $\pm$ 0.06 & \textbf{0.38 $\pm$ 0.06} & 0.54 $\pm$ 0.19 &\textbf{ 0.75 $\pm$ 0.32} & 14.88  \\ 
    \textbf{PostDoc} & \textbf{0.68 $\pm$ 0.14} & \textbf{0.70 $\pm$ 0.10} & \textbf{0.36 $\pm$ 0.13} & 0.30 $\pm$ 0.03 & \textbf{0.58 $\pm$  0.06} & 0.74 $\pm$ 0.34 & \textbf{0.68}  \\
    \bottomrule
\end{tabular}
}
\label{tab:MSMOmaintable}
\end{table*}

\begin{table*}[ht]
\centering
\caption{Comparison of various multimodal summarization methods on NJU-Fudan dataset~\citep{qiang2019learning}}
\resizebox{0.9\linewidth}{!}{
\begin{tabular}{c|cccccccc}
    \toprule
      Method & ROUGE-L & ROUGE-1 & ROUGE-2 & Coverage & Diversity & Image Precision & Inference time (sec)\\
    \midrule
    MemSum + BLIP & 0.27 $\pm$ 0.03 & 0.27 $\pm$ 0.03 & 0.17 $\pm$ 0.03 & 0.38 $\pm$ 0.05 & 0.64 $\pm$ 0.05 & 0.39 $\pm$ 0.28 & 10.35  \\  
    BRIO + BLIP & 0.07 $\pm$ 0.02 & 0.07 $\pm$ 0.02 & 0.03 $\pm$ 0.01 & 0.27 $\pm$ 0.06 & \textbf{0.66 $\pm$ 0.06} & 0.38 $\pm$ 0.26 & 6.09  \\
    GPT-3.5 + BLIP & 0.13 $\pm$ 0.05 & 0.13 $\pm$ 0.05 & 0.06 $\pm$ 0.04 & 0.31 $\pm$ 0.06 & 0.65 $\pm$ 0.05 & 0.39 $\pm$ 0.28 & 47.37  \\ 
    \textbf{PostDoc} & \textbf{0.50 $\pm$ 0.04}  & \textbf{0.48 $\pm$ 0.03} & \textbf{0.33 $\pm$ 0.03} & \textbf{0.42 $\pm$ 0.03} & 0.61 $\pm$ 0.04 & \textbf{0.53 $\pm$ 0.32} & \textbf{4.25} \\
    \bottomrule
\end{tabular}
}
\label{tab:Fudanmaintable}
\end{table*}

\begin{table*}[ht]
\centering
\caption{Model ablation study of PostDoc on MSMO and NJU-Fudan Datasets}
\resizebox{\linewidth}{!}{
\begin{tabular}{c|cccc|cccc}
    \toprule
       & \multicolumn{4}{c}{MSMO Datset} & \multicolumn{4}{c}{NJU-Fudan Dataset} \\
      Method & ROUGE-L & Coverage & Diversity & Image Precision & ROUGE-L & Coverage & Diversity & Image Precision  \\
    \midrule
    PostDoc w/o dsf & 0.57 $\pm$ 0.13 & 0.23 $\pm$ 0.04 & 0.61 $\pm$ 0.09 & 0.71 $\pm$ 0.44 & 0.43 $\pm$ 0.10 & 0.34 $\pm$ 0.07 & 0.65 $\pm$ 0.07 & 0.37 $\pm$ 0.17 \\ 
    PostDoc w/o coverage & 0.51 $\pm$ 0.08 & 0.24 $\pm$ 0.06 & 0.61 $\pm$ 0.14 & \textbf{0.75 $\pm$ 0.27} & 0.48 $\pm$ 0.05 & \textbf{0.43 $\pm$ 0.04} & 0.57 $\pm$ 0.06 & 0.52 $\pm$ 0.32 \\
    PostDoc w/o diversity & 0.50 $\pm$ 0.13 & 0.24 $\pm$ 0.03 & 0.62 $\pm$ 0.14 & 0.75 $\pm$ 0.27 & 0.48 $\pm$ 0.05 & 0.43 $\pm$ 0.04 & 0.58 $\pm$ 0.06 & 0.53 $\pm$ 0.32 \\
    PostDoc w/o alignment & 0.56 $\pm$ 0.12 & 0.25 $\pm$ 0.06 & \textbf{0.63 $\pm$ 0.13} & 0.75 $\pm$ 0.28 & 0.45 $\pm$ 0.04 & 0.36 $\pm$ 0.03 & \textbf{0.67 $\pm$ 0.03} & 0.34 $\pm$ 0.19 \\
    \textbf{PostDoc} & \textbf{0.68 $\pm$ 0.14} & \textbf{0.30 $\pm$ 0.03} & 0.58 $\pm$  0.05 & 0.74 $\pm$ 0.34 & \textbf{0.50 $\pm$ 0.04} & 0.42 $\pm$ 0.03 & 0.61 $\pm$ 0.04 & \textbf{0.53 $\pm$ 0.32}  \\
    \bottomrule
\end{tabular}
}
\label{tab:ablation}
\end{table*}

\begin{table}[ht]
\centering
\caption{Results of font recommendation on Let Me Choose Dataset and conditional layout generation on the NJU-Fudan Dataset}
\resizebox{\linewidth}{!}{
\begin{tabular}{ccc|cc}
    \toprule
      \multicolumn{3}{c}{Font Recommendation} & \multicolumn{2}{c}{Layout Generation} \\
      Method & Top-1 F1 & Top-3 F1 & Method & NGOMetric \\
    \midrule
    BERT Model & 0.2697 & 0.5191 & LayoutDM & 0.27 \\
    PostDoc & \textbf{0.4301} & \textbf{0.5950} & PostDoc & \textbf{0.46} \\
    \bottomrule
\end{tabular}
}
\label{tab:ablation_layout_font}
\end{table}

    


   

\section{Experimental Analysis}\label{sec:exp}
In this section, we discuss the details about the experimental setup and results obtained from both automated and human evaluation to understand the quality of the posters generated from PostDoc.

\subsection{Datasets Used}\label{sec:dataset}
We use the MSMO Dataset collected by \citet{zhu-etal-2018-msmo} for training and testing our method for multimodal summarization. The MSMO Dataset has 312,581 samples of which only the test set (10,261 samples) has image annotations present as part of the multimodal summary. We require image annotations for training our deep submodular function. So, we use 9000 samples from the test set for our training, 261 samples for validation and the remaining 1000 samples for testing. We report the performance of our multimodal summarization method on these 1000 samples in Table ~\ref{tab:MSMOmaintable}.

In order to test the content generated by our multimodal summarization module to that with the actual posters, we make use of the NJU Fudan Dataset \citep{qiang2019learning}. It contains 85 pairs of scientific papers and their corresponding posters. We extract the text and images from the papers and posters using Adobe Extract. We filter paper-poster pairs so that they each have at least one image after being processed by Extract. After this step, we have a filtered dataset of 76 paper-poster pairs. We do not use any portion of this dataset for training. We use it only for testing our summarization method on out of domain data to analyze the generalization ability in Table \ref{tab:Fudanmaintable}.

For training the font selection model, we make use of the Let Me Choose Dataset ~\citep{shirani2020font} as discussed in Section \ref{sec:tempgen}. It contains 1309 short texts which are mapped to one of 10 fonts. We follow the same split mentioned by the authors for training (70\%),validation (10\%) and testing (20\%). 

\subsection{Baseline Methods}
\label{sec:baseline}
We could not find any replicable source code for existing document-to-poster works \cite{qiang2019learning,xu2021neural,xu2022posterbot} to make an end-to-end comparison. So, for a thorough evaluation, we analyze each module of PostDoc with the corresponding baselines.
\paragraph{Multimodal Summarization}
To the best of our understanding, there is not publicly available replicable implementation for any of the existing multimodal-in and multimodal-out summarization approaches \citep{zhu-etal-2018-msmo,Zhu_Zhou_Zhang_Li_Zong_Li_2020,zhang2021hierarchical,zhang2022unims}.
Additionally, as we are using a subset of the MSMO test set (\S\ref{sec:dataset}) for training, we will not be able to carry over the metrics reported in these works. 
So, we compare our performance with baselines for text summarization and include images in the summary as a post-processing step.
For the text summarization, we use the following.\\
1. \textbf{Extractive Summarization}: We use the publicly available \textbf{MemSum} architecture ~\citep{gu2021memsum} which currently has the SOTA performance on the GovReport dataset ~\citep{huang2021efficient}. \\
2. \textbf{Abstractive Summarization}: We use the publicly available \textbf{BRIO} ~\citep{liu2022brio} architecture, which currently has the SOTA performance on the CNN/Daily Mail dataset ~\citep{nallapati-etal-2016-abstractive} \\
3. \textbf{GPT-3.5-turbo}: In this baseline, we chunk the input text and summarize each chunk with GPT-3.5-turbo. We concatenate the summaries and finally paraphrase it further with another GPT-3.5-turbo call.

With each of the above text summaries, the images are included in the summary by choosing the top $K_I$ images from the input document based on their similarity with the summarized text. Here the similarity is calculated by the cosine of the BLIP embeddings \citep{li2022blip} of text and images. From the training set, we calculate the average ratio of the number of images in the summary and that in the input document. To fix $K_I$ during inference on a document, we multiply the number of images present in the document with that ratio.
\paragraph{Template Generation}
For font selection, we compare our method (\S\ref{sec:layout_pipeline})
with the Let Me Choose BERT Model ~\citep{shirani2020font} on the Let Me Choose Dataset as shown in Table ~\ref{tab:ablation_layout_font}.
For layout generation, we compare our method (\S\ref{sec:layout_pipeline}) with LayoutDM \citep{inoue2023layout} on the NJU Fudan Dataset, in Table~\ref{tab:ablation_layout_font}.




\subsection{Evaluation Metrics}
We evaluate the generated posters on the following aspects.

\subsubsection{Multimodal Summarization}
To evaluate the salience of the text generated by our method, we employ the standard text summarization metric ROUGE-1, ROUGE-2 and ROUGE-L which compares the generated summary with the ground truth.
To evaluate the generated multimodal summary, we use coverage and diversity \citep{kothawade2020deep} along with image precision and image recall.
Coverage is measured between the generated multimodal summary and the multimodal source document.
\[ Coverage(A) = \frac{1}{|D||A|}\sum_{x \in D}\sum_{y \in A}cosine\bigl(x, y\bigr) \]
\[ Diversity(A) = 1 - \frac{1}{|A|^2}\sum_{x, y \in A}cosine\bigl(x, y\bigr) \] 
Here, $D$ contains the BLIP embeddings for all the content elements of the input multimodal document and $A$ contains the BLIP embeddings from the generated summary.
An ideal summary will have high coverage and high diversity.
Following \citet{zhu-etal-2018-msmo},  we use image precision \textit{$I_P$} to evaluate the images selected by our method. 
$I_P=\frac{|I_A \cap I_G|} {|I_A|}$, where $I_A$ and $I_G$ refer to the images in generated summary and the images in the ground truth summary. We also report average inference time as a metric for computational overhead.



\subsubsection{Template Generation}
\textbf{Font Selection}:
Following ~\citet{shirani2020font}, we measure the performance of font selection models using the average weighted F1-Score over top $k$ where $k$ = \text{\{1, 3\}}.\\
\textbf{Layout Generation}:
To evaluate the layouts generated, we use a combination of the following NGOMetrics~\footnote{\url{http://www.mi.sanu.ac.rs/vismath/ngo/index.html}}:
We use a weighted combination of the equilibrium of the bounding boxes, padding in the layout, density of the bounding boxes with respect to the overall layout, and overlap between the bounding boxes. This allows us to score layouts based on their aesthetic properties.
Please refer to the details about the computation of these metrics in the supplementary material.

\subsection{Performance Analysis}
\paragraph{Multimodal Summarization}
Tables \ref{tab:MSMOmaintable} and \ref{tab:Fudanmaintable} show the mean and standard deviation of the performance of PostDoc and the baselines on MSMO and NJU Fudan datasets respectively. By investigating the results, we note the following:
\textbf{(1)}: PostDoc outperforms all of the baselines on the basis of all the ROUGE metrics with significant margins. This shows that the our multimodal summarization module captures relevant information from the source document which aligns with the ground truth.
\textbf{(2)}: Regarding the visual modality metric image precision, our model performs on par or worse than the baselines on MSMO dataset, but significantly outperforms all the baselines on NJU-Fudan dataset.
It is important to note that the number of images selected in the baseline is pre-fixed based on the average number of images present in the document on the training set. However, for PostDoc, we do not differentiate between the text and images in the selection criteria during the inference. This gives a benefit to the baselines on the MSMO dataset for image selection.
\textbf{(3)}: As mentioned in Section \ref{sec:baseline}, we consider extractive summarization (MemSum), abstractive summzarization (BRIO) and GPT-3.5-turbo for the text summarization module of the baselines.
MemSum and BRIO  are trained on 17517 and 200000 samples respectively and GPT-3.5-turbo is trained on large text databases available on the internet. Our model is able to perform competitively with these baselines even though the data it is trained on (9000 samples) is significantly lesser than the training data of these text summarization methods.
\textbf{(4)}: The GPT-3.5-turbo+BLIP baseline performs competitively on coverage and diversity. But the latency and the cost associated with GPT calls is very high. On average, the number of tokens processed by GPT-3.5 for MSMO dataset is 838.19 and the NJU Fudan dataset is 5323.85. PostDoc is much faster compared to all the baselines in terms of average inference time on a document.

\paragraph{Template Generation}
Table \ref{tab:ablation_layout_font} shows the performance of our model and the baseline on the Let Me Choose dataset for font selection. Our model outperforms the baseline on both the Top-1 F1 Score and the Top-3 F1 Score.
Table \ref{tab:ablation_layout_font} also compares our layout generation module with LayoutDM. As LayoutDM wasn’t explicitly trained on posters, we initially generate 250 candidate layouts using LayoutDM and choose the layout which gives the highest equi-weighted NGOMetrics Score. Our layout generation module, which relies on heuristics, achieves a score that is almost twice the score achieved by LayoutDM.

\subsection{Model Ablation Study}
To understand the importance of different components of PostDoc, we perform the following ablation experiments on MSMO and NJU Fudan datasets, and report the results in Table ~\ref{tab:ablation}.

\textit{PostDoc w/o DSF}: Instead of using the deep submodular function proposed in Equation ~\ref{eq:dsf}, we make use of a simple feature-based submodular function (without any trainable parameters) which takes coverage and diversity into account: $f(A)=\lambda\sum_{x \in A}\sum_{y \in D}x_u y_u - \sum_{x \in A}\sum_{y \in A} x_u y_u$. For inference we choose the subset $A$ which gives the maximum value of $f(A)$. The value of \text{$\lambda$} was set as 0.2 for this experiment.

\textit{PostDoc w/o Coverage}: We remove the first term  $\sum_{x \in A}\sum_{y \in D}x_u y_u$ from Equation~\ref{eq:dsf} and proceed with the min-max optimization procedure mentioned in Section ~\ref{sec:opti}.

\textit{PostDoc w/o Diversity}: We remove the second term  $\sum_{x \in A}\sum_{y \in A} x_u y_u$  from Equation~\ref{eq:dsf} and proceed with the min-max optimization procedure mentioned in Section ~\ref{sec:opti}.

\textit{PostDoc  w/o Alignment}: We remove the third term  $\sum_{x \in A_I}\sum_{y \in A_T} x_u y_u$ from Equation~\ref{eq:dsf} and proceed with the min-max optimization procedure mentioned in Section ~\ref{sec:opti}.

From Table \ref{tab:ablation}, we can see that PostDoc is able to achieve the best performance on Rouge metrics which shows the importance of the combination of all the components present in it. On the other three metrics, we do not see any consistent pattern among the different model variants. 

\begin{table}[ht]
\centering
\resizebox{0.8\linewidth}{!}{
\begin{tabular}{c|cc}
    \toprule
       Questions & GPT-3.5 + BLIP & PostDoc \\
    \midrule
    Coverage & 3.13 $\pm$ 0.83 & \textbf{3.40 $\pm$ 0.64} \\
    Duplication & \textbf{4.26 $\pm$ 0.27} & 4.13 $\pm$ 0.37 \\
    Content Ordering & 3.33 $\pm$ 0.62 & 3.33  $\pm$ 0.47 \\
    Image Selection & 2.66 $\pm$ 0.47 & \textbf{3.0 $\pm$ 0.70} \\
    Template & 2.33 $\pm$ 0.5 & \textbf{3.46 $\pm$ 0.18} \\
    Run Time & 1.86 $\pm$ 0.29 & \textbf{3.60 $\pm$ 0.64} \\
    \bottomrule
\end{tabular}
}
\caption{Human evaluation on a subset of NJU-Fudan Dataset}
\label{tab:humanEval}
\end{table}

\subsection{Human Evaluation}
We have conducted a small scale human survey to understand the quality of the generated posters from the actual human perspective. We hired 3 experts in AI as reviewers for this task. We randomly chose 5 research papers from NJU Fudan dataset. We used GPT-3.5+BLIP as the only baseline for this study as GPT-3.5 is often considered to be close to humans in generative tasks. Each reviewer generated the posters by the two algorithms from each of the selected 5 documents and gave a rating on a scale of 1 (worst) to 5 (best) to each poster for each of the following questions:
(1) How good is the poster to \textit{cover} the document?
(2) Is there any \textit{duplication} of content in the poster?
(3) How good is the \textit{ordering of content} in the poster?
(4) How good are the \textit{selected images} in the poster?
(5) How good is the \textit{template} of the poster?
(6) Are you satisfied with the \textit{run time} of the algorithm?

In Table \ref{tab:humanEval}, we compute average and standard deviation of the ratings provided by all the reviewers on all the documents. Interestingly, the human evaluation results correlate well with the automated evaluation on NJU Fudan dataset in Table \ref{tab:Fudanmaintable}. The human evaluation shows that PostDoc is as per or better in terms of output content quality and much better in terms of user satisfaction with the layout and runtime.



\section{Discussion and Conclusion} 
We have presented PostDoc, an end-end pipeline to automatically generate a poster from a long multimodal document. It is interesting to find that PostDoc, by using a novel combination of deep submodular functions and a single call to LLM (ChatGPT) is able to achieve better or comparable performance than direct calls to the same LLM which is expensive both in terms of computation and cost.
In the current work, the performance of PostDoc is limited for non-natural images such as flow-chart and neural diagrams, and other structured elements like tables, etc. We plan to fine-tune VLMs on documents containing such elements as a future work.






\bibliography{PostDoc}

\appendix{}

\section{Template Generation for PostDoc}
In this section, we provide the details of the color selection and layout generation modules which we had to skip in the main paper for the limited space.


\subsection{Color Selection}
Colors are also important for posters as they capture attention and contextualise the content~\footnote{\url{https://www.snap.com.au/blog/colours-and-advertising-posters}}. Our pipeline has three main colors: 
\begin{itemize}
    \item Background color - that underpins the poster
    \item Box fill color - that serves as the color for the bounding boxes that house textual content
    \item Text fill color - the font color of the texts themselves.
\end{itemize}

To select the color palette for the poster, we used the Text2Colors' TPN model \cite{text2colors} to generate a candidate palette. \\ \\ From this, we used the relative CLAB ratios to find the most contrastive color within the palette and colored the background of the poster with this color. We then colored the text bounding boxes with its complement. \\ \\




The \textit{dominant color} in the palette is chosen to be the background color of the poster, and its \textit{complement} is chosen to be the box fill color. The text text fill color is chosen to be black or white based on its contrast with the box fill color. We then use the background color in a prompt to Firefly~\footnote{\url{https://firefly.adobe.com/}} that generates a background grounded on it. 

For this, we make use of \textit{thecolorapi.com} to get a textual name for the background/dominant color. \\

We now describe the equation that governs the selection of the dominant color. Given two colors $C$ and $C'$, the relative luminance of $C$ w.r.t $C'$ is \\
\[RL(C, C') = \frac{max(L(C), L(C') + 0.05}{min(L(C), L(C') + 0.05}\]
where $L(C)$ is the luminance (perceived brightness) of the RGB color $C$ as defined in the WCAG standard~\footnote{\url{https://www.w3.org/WAI/GL/wiki/Relative_luminance}}  
\\ \\ 
Given a set of colors $\mathbb{C}$, dominant color $C_d$ is calculated as follows: \\

\[C_d = max_{C \in \mathbb{C}} \sum_{\substack{C' \in \mathbb{C} \\ C' \neq C}}{} RL(C, C')\]

\subsection{Layout Generation}
We propose a heuristic based approach for generating a balanced layout conditioned on the paraphrased content. For each topic with the associated bullet points from the paraphrased content, we create a text box in the poster layout. Similarly, for each image (with the associated caption when available), we create an image box. Let us denote $N_T$ and $N_I$ as the required number of text and image boxes in the layout respectively. We begin by fixing the bounding box for the title, leaving the remaining space ($l$) for distributing image and text boxes. We kept the images on the left and text boxes on the right by dividing the space into two parts vertically ($b1$ and $b2$). The width of the text boxes is fixed (say $b2 - \alpha$) where as that of the images is adjusted based on number of images (say $b1 - \beta /N_I$). Here, $\alpha$ and $\beta$ are  to ensure some marginal gaps between the boxes and the page margin. The width of the captions box is same as that of images. To estimate the height of the text boxes, we consider the content length and the $N_T$ and $N_I$ for captions.  For images, $height = width/aspect\_ratio$. The distance between the text boxes ($dh_1$) and between the image boxes ($dh_2$) is calculated depending on the values of $N_T$ and $N_I$, respectively. Refer to the figure \ref{fig:layout} for better understanding.

\begin{figure}
    \centering
    \includegraphics[width = 0.5\textwidth]{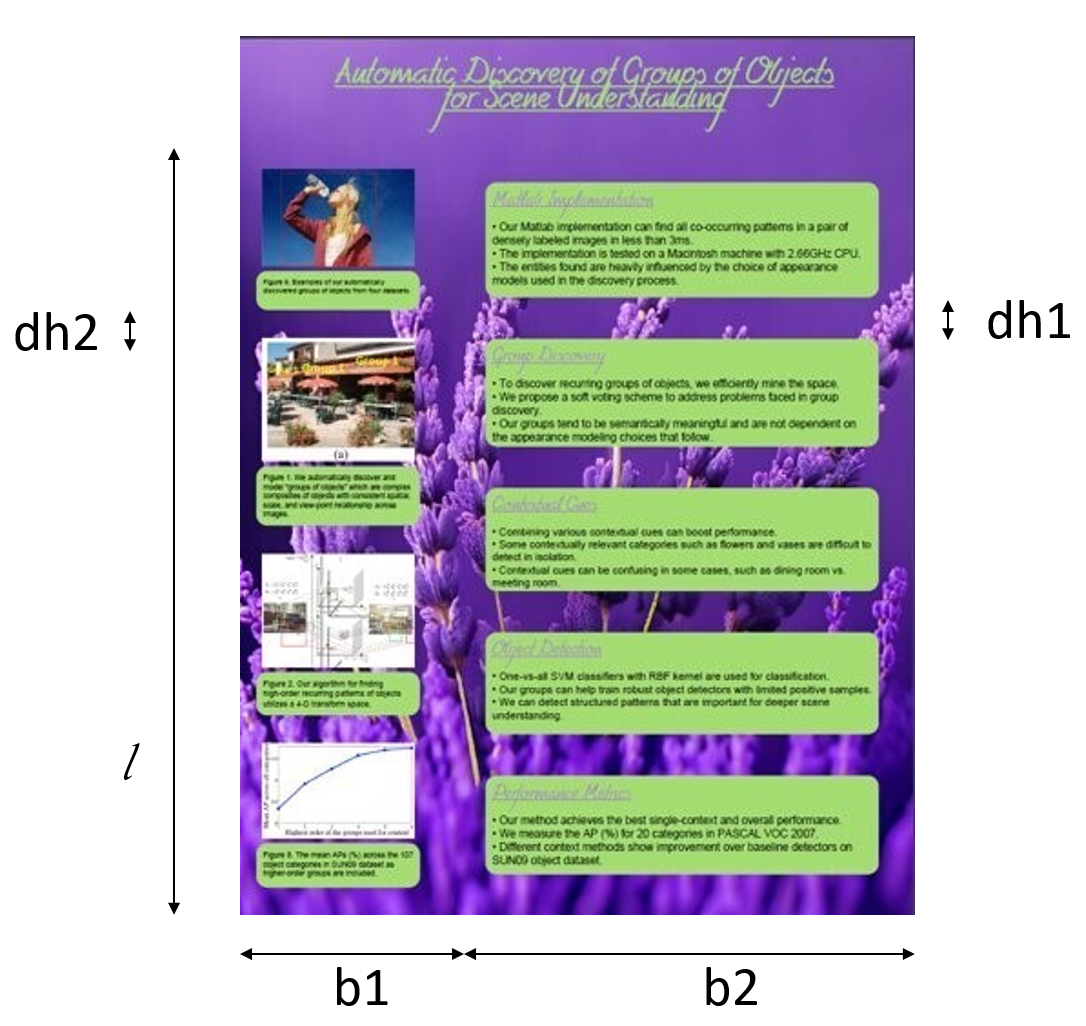} 
    \caption{A sample layout generated by this method ($N_I$ = 4, $N_T$ = 5)}
    \label{fig:layout}
\end{figure}

\begin{table*}[ht]
\centering
\resizebox{\linewidth}{!}{
\begin{tabular}{c|cccccc}
    \toprule
      Method & ROUGE-L & ROUGE-1 & ROUGE-2 & Coverage & Diversity & Image Precision  \\
    \midrule
    PostDoc with submodular function & 0.579509 $\pm$ 0.130993 & 0.465446 $\pm$ 0.051633 & 0.276391 $\pm$ 0.028582 & 0.228677 $\pm$ 0.044838 & 0.617691 $\pm$ 0.093911 & 0.716180 $\pm$ 0.437901 \\
    PostDoc optimized w/o coverage & 0.512333 $\pm$ 0.086352 & 0.526364 $\pm$ 0.053763 & 0.248168 $\pm$ 0.030499 & 0.248681 $\pm$ 0.064308 & 0.619039 $\pm$ 0.146075 & \textbf{0.753445 $\pm$ 0.275050} \\
    PostDoc optimized w/o diversity & 0.507014 $\pm$ 0.130131 & 0.521063 $\pm$ 0.055315 & 0.245248 $\pm$ 0.030504 & 0.248617 $\pm$ 0.248617 & 0.617401 $\pm$ 0.146460 & 0.752585 $\pm$ 0.276562 \\
    PostDoc optimized w/o alignment & 0.562540 $\pm$ 0.129673 & 0.577946 $\pm$ 0.047215 & 0.276391 $\pm$ 0.033015 & 0.252465 $\pm$ 0.064187 & \textbf{0.634624 $\pm$ 0.135463} & 0.748600 $\pm$ 0.282186  \\
    \textbf{PostDoc} & \textbf{0.67703 $\pm$ 0.142195} &\textbf{ 0.703324 $\pm$ 0.100135} & \textbf{0.357304 $\pm$ 0.135787} & \textbf{0.296999 $\pm$ 0.032629} & 0.576494 $\pm$  0.059301 & 0.739570 $\pm$ 0.341569 \\
    \bottomrule
\end{tabular}
}
\caption{Ablation study of PostDoc on MSMO Dataset}
\label{tab:ablationMSMO}
\end{table*}

\begin{table*}[ht]
\centering
\resizebox{\linewidth}{!}{
\begin{tabular}{c|cccccc}
    \toprule
      Method & ROUGE-L & ROUGE-1 & ROUGE-2 & Coverage & Diversity & Image Precision \\
    \midrule
    PostDoc with submodular function & 0.426815 $\pm$ 0.103294 &\textbf{0.514749 $\pm$ 0.044017} & \textbf{0.384185 $\pm$ 0.044100} & 0.343835 $\pm$ 0.065151 & 0.651721 $\pm$ 0.078240 & 0.366457 $\pm$ 0.174725 \\  
    PostDoc optimized w/o coverage & 0.475431 $\pm$ 0.050257 & 0.455375 $\pm$ 0.048136 & 0.325059 $\pm$ 0.047058 & \textbf{0.432243 $\pm$ 0.046794} & 0.577893 $\pm$ 0.065549 & 0.529320 $\pm$ 0.321356 \\
    PostDoc optimized w/o diversity & 0.486038 $\pm$ 0.052330 & 0.465534 $\pm$ 0.050122 & 0.332959 $\pm$ 0.047065 & 0.431614 $\pm$ 0.046335 & 0.580118 $\pm$ 0.065608 & 0.530137 $\pm$ 0.320465 \\
    PostDoc optimized w/o alignment & 0.455443 $\pm$ 0.046483 & 0.503130 $\pm$ 0.051569 & 0.362864 $\pm$ 0.050940 & 0.361852 $\pm$ 0.039532 & \textbf{0.677698 $\pm$ 0.031189} & 0.343208 $\pm$ 0.194419 \\
    \textbf{PostDoc} & \textbf{0.503238 $\pm$ 0.041803} & 0.482009 $\pm$ 0.036599 & 0.331308 $\pm$ 0.034393 & 0.421352 $\pm$ 0.031297 & 0.609323 $\pm$ 0.043384 & \textbf{0.531805 $\pm$ 0.325263} \\
    \bottomrule
\end{tabular}
}
\caption{Ablation study of PostDoc on NJU-Fudan~\cite{qiang2019learning} Dataset}
\label{tab:ablationFudan}
\end{table*}

\begin{table*}[ht]
\centering
\resizebox{0.75\linewidth}{!}{
\begin{tabular}{c|cccc}
    \toprule
      Method & Equilibrium score & Padding score & Density score & Overlap score \\
    \midrule
    
    LayoutDM & 0.450 & 0.282 & 0.092 & 0.266  \\
    PostDoc & 0.600 &  0.213 &  0.256 &  0.774 \\ 

    \bottomrule
\end{tabular}
}
\caption{Comparison study of the conditional layout generation models on the \textit{test} section of NJU-Fudan dataset}
\label{tab:ablation_layout_ngometrics}
\end{table*}

However, when the number of images is less than three, we elongate the bottom text boxes towards the image (left) side and reduce $dh_2$ to ensure visual balance.
$dh_1$ and $dh_2$ are calculated as shown below:
\[ dh_1 = \frac{l - \sum_i h_T^i}{N_T+1} \]
\[ if \ N_I > 3, \ dh_2 = \frac{l - \sum_i {(h_I^i + h_C^i)} - k_1 \cdot N_I}{N_I+1} \]
\[else \ if \ N_I <= 2, \ dh_2 = dh_1 \]
  
Once we get the heights and widths of the bounding boxes, we can use them to calculate the top left vertex of the bounding box. Their Y coordinates of these points are calculated as shown:
\[ Y_T^i = l - (i \cdot dh_1) - \sum_{j=1}^{i-1} h_T^j \]
\[ Y_I^i = l - (i \cdot dh_2) - \sum_{j=1}^{i-1}(h_I^j + h_C^j) -(i-1) \cdot k_2 \]
\[ Y_C^i = Y_I^i -h_I^i - k_2 \]
where $h_T^i$, $h_I^i$, $h_C^i$ are the estimated heights of the $i^{th}$ text, image and caption box respectively. $k_1$ is added to make $dh_2$ depend on the number of images. $k_2$ is the gap between the image and caption boxes. The X coordinates of these points can be linearly interpolated based on the widths of the bounding boxes, $b1$, and $b2$.  Once we get these cordinates , height, and width, we can calculate the other cordinates.
By following this approach, we achieve a well-organized and visually appealing layout for posters, adapting the design based on the number of text and image boxes available.

\section{Additional Results for Model Ablation Study}






Table ~\ref{tab:ablationMSMO} and Table ~\ref{tab:ablationFudan} show the performance of model variants on all the metrics we used in the other experiments. 
\footnote{Please note that, due to a lack of space in the main paper, we did not include the ROUGE-1 and ROUGE-2 score in the model ablation tables of the main paper.}  
Our model, PostDoc, performs best against the ablations on the MSMO test dataset in terms of the ROUGE-1 and ROUGE-2 scores as well.

\section{Details of NGOMetrics}
The metric to evaluate the quality of the generated layouts was built on top of design best practices, as defined by NGO-Metrics~\footnote{\url{http://www.mi.sanu.ac.rs/vismath/ngo/index.html}}. 
%
We use a combination of selected processed NGOMetrics.

We now formally define the NGOMetrics' scoring mechanisms and the notations for the same
\begin{itemize}
    \item The overall frame has a width $W$ and height $H$ \\ $\Rightarrow$ The area of the whole frame is $W \cdot H$
    \item Each rectangular bounding box $i$ has an area $a_i$, with x-coordinates ranging from $x_1^{i}$ to $x_2^{i}$, and y-coordinates ranging from $y_1^{i}$ to $y_2^{i}$, so its center of mass is  $[COM^{i}_{x}, COM^{i}_{y}] $ = $\frac{x_1^{i} + x_2^{i}}{2}, \frac{y_1^{i} + y_2^{i}}{2}$
    \item We condition the layouts on a total of $\mathbb{B}$ rectangular bounding boxes.
\end{itemize}

\vspace{5pt}

\textbf{Equilibrium}
This represents the distance of the center of mass of the bounding boxes (when taken together) and the center of mass of the layout. \\
This metric is calculated as \\ \\ \[equilibrium = 1 - \frac{|EM|_x + |EM|_y}{2}\] where \\ \\
\[EM_x = \frac{2 \cdot \sum_{i=1}^{i=\mathbb{B}} [a_i \cdot (COM^{i}_{x} - W/2)]}{\mathbb{B} \cdot W \cdot \sum_{i=1}^{i=\mathbb{B}}a_i}\] and \\ \\
\[EM_y = \frac{2 \cdot \sum_{i=1}^{i=\mathbb{B}} [a_i \cdot (COM^{i}_{y} - H/2)]}{\mathbb{B} \cdot H \cdot \sum_{i=1}^{i=\mathbb{B}}a_i}\]

\vspace{10pt}

\textbf{Padding}: This represents the area between the boundary of the layout and the topmost and rightmost edges of the bounding boxes. \\  
This metric is calculated as \\ \\ \[padding = 1 - \frac{[R\_max - L\_min] \cdot [T\_max - B\_min]}{W \cdot H}\] where \\ 
\[R\_max = max_{1 \leq i \leq \mathbb{B}} [x_2^{i}]\]
\[L\_min = min_{1 \leq i \leq \mathbb{B}} [x_1^{i}]\]
\[T\_max = max_{1 \leq i \leq \mathbb{B}} [y_2^{i}]\]
\[B\_min = min_{1 \leq i \leq \mathbb{B}} [y_i^{i}]\]
    
\vspace{10pt}

\textbf{Density}: This represents the proportion of the frame layout covered by the bounding boxes. \\
This metric is calculated as \\ \\ \[density = max(1, \frac{\sum_{i=1}^{i=\mathbb{B}} a_i}{W \cdot H \cdot \mathbb{B}})\]    
\vspace{10pt}

\textbf{Overlap}: This represents the sum of areas of intersection between the bounding boxes in the layout. 

This metric is calculated as \\ \\ \[overlap = 1 - max(1, \frac{\sum_{i=1}^{i=\mathbb{B}-1} \sum_{j=i+1}^{j=\mathbb{B}} O_{eff}(i, j))}{W \cdot H})\] where \\
\[O_{eff}(i, j) = \mathbb{1}_{\Delta x(i, j) > 0} \cdot \mathbb{1}_{\Delta y(i, j) > 0} \cdot \Delta x(i, j) \cdot \Delta y(i, j)\]  having \\
\[\Delta x(i, j) = min[x_2^{i}, x_2^{j}] - max[x_1^{i}, x_1^{j}]\] 
\[\Delta y(i, j) = min[y_2^{i}, y_2^{j}] - max[y_1^{i}, y_1^{j}]\]

\vspace{10pt}

\textbf{Overall scoring function} The \textit{overall scoring function} is as follows: 
$w = 0.25 \cdot [equilibrium] + 0.25 \cdot [padding] + 0.25 \cdot [density] + 0.25 \cdot [overlap]$ \\ \\


We now present the average scores of these NGO Metrics for the test section of the NJU-Fudan dataset. The conditions for the LayoutDM model \citep{chai2023layoutdm} was the \textit{(number of topics provided by GPT 3.5 after paraprhasing the textual content + number of images selected)} text boxes + \textit{(number of images selected)} image boxes.


While LayoutDM works with conditional inputs, our observations was that with increased number of text and image boxes, the overlap increases and the layouts look more cluttered, particularly near the center of the document. Our algorithm scales better with an increase in inputs in lieu of density and overlap scores.

Future work includes finetuning LayoutDM on posters and new conventions to order images and text bounding boxes. Furthermore, papers like \cite{Hsu_2023_CVPR} which generate layouts from an image but do not explicitly take number of bounding boxes as inputs, so we are not using it in our work.

\end{document}